\documentclass[journal]{IEEEtai}
\usepackage{xcolor,soul,framed} %,caption

\colorlet{shadecolor}{yellow}
\usepackage[pdftex]{graphicx}
\graphicspath{{../pdf/}{../jpeg/}}
\DeclareGraphicsExtensions{.pdf,.jpeg,.png}

\usepackage[cmex10]{amsmath}
\usepackage{amssymb}
\usepackage{multirow}
\usepackage{graphicx}
\usepackage{array}
\usepackage{mdwmath}
\usepackage{mdwtab}
\usepackage{eqparbox}
\usepackage{algorithm}
\usepackage{algorithmicx}
\usepackage{algpseudocode}
\usepackage{amsmath}
\usepackage{amsfonts}
\usepackage{amsthm}   % For theorem-like environments

\usepackage[noadjust]{cite}

\newtheorem{proposition}{Proposition}[section] % Number propositions by section
\usepackage{mathrsfs}
\usepackage{url}
\usepackage[colorlinks,urlcolor=blue,linkcolor=blue,citecolor=blue]{hyperref}
\usepackage[switch]{lineno}
\usepackage[moderate,tracking=normal,paragraphs=normal]{savetrees} % paragraphs=normal,tracking=normal,

\begin{document}
\bstctlcite{IEEEexample:BSTcontrol}
\title{Lorentz-Equivariant Quantum Graph Neural Network for High-Energy Physics}
\author{Md~Abrar~Jahin, 
    Md.~Akmol~Masud, 
    Md~Wahiduzzaman~Suva, 
    M.~F.~Mridha,~\IEEEmembership{Senior Member,~IEEE,}
    and Nilanjan~Dey,~\IEEEmembership{Senior Member,~IEEE}

\thanks{M. A. Jahin is with Okinawa Institute of Science and Technology Graduate University, Okinawa 904-0412, Japan (e-mail: abrar.jahin.2652@gmail.com).}% <-this % stops a space
\thanks{M. A. Masud is with the Institute of Information Technology, Jahangirnagar University, Dhaka, 1342, Bangladesh (e-mail: akmolmasud5@gmail.com ).}%
\thanks{M. W. Suva and M. F. Mridha are with the Department of Computer Science, American International University-Bangladesh, Dhaka 1229, Bangladesh (e-mail: wahedshuvo36@gmail.com; firoz.mridha@aiub.edu).}% <-this % stops a space
\thanks{N. Dey is with the Department of Computer Science \& Engineering, Techno International New Town, New Town, Kolkata, 700156, India (e-mail: nilanjan.dey@tint.edu.in).}}  

%%% ORCID %%%%
% 0000-0002-1623-3859  <- Abrar
% 0009-0009-1894-9533 <- Akmol
% 0009-0007-6227-7282 <- Suva

% \linenumbers

% \markboth{IEEE TRANSACTIONS ON NEURAL NETWORKS AND LEARNING SYSTEMS
% }{Jahin \MakeLowercase{\textit{et al.}}: Lorentz-EQGNN}
\markboth{
}{Jahin \MakeLowercase{\textit{et al.}}: Lorentz-EQGNN}

% \linenumbers

% ====================================================================
\maketitle

% === ABSTRACT ====================================================================
% =================================================================================
\begin{abstract}
% 150-200 words and concise
The rapid data surge from the high-luminosity Large Hadron Collider introduces critical computational challenges requiring novel approaches for efficient data processing in particle physics. Quantum machine learning, with its capability to leverage the extensive Hilbert space of quantum hardware, offers a promising solution. However, current quantum graph neural networks (GNNs) lack robustness to noise and are often constrained by fixed symmetry groups, limiting adaptability in complex particle interaction modeling. This paper demonstrates that replacing the classical Lorentz Group Equivariant Block modules in LorentzNet with a dressed quantum circuit significantly enhances performance despite using $\approx5.5$ times fewer parameters. Additionally, quantum circuits effectively replace MLPs by inherently preserving symmetries, with Lorentz symmetry integration ensuring robust handling of relativistic invariance. Our \ul{Lorentz}-\ul{E}quivariant \ul{Q}uantum \ul{G}raph \ul{N}eural \ul{N}etwork (Lorentz-EQGNN) achieved 74.00\% test accuracy and an AUC of 87.38\% on the Quark-Gluon jet tagging dataset, outperforming the classical and quantum GNNs with a reduced architecture using only 4 qubits. On the Electron-Photon dataset, Lorentz-EQGNN reached 67.00\% test accuracy and an AUC of 68.20\%, demonstrating competitive results with just 800 training samples. Evaluation of our model on generic MNIST and FashionMNIST datasets confirmed Lorentz-EQGNN’s efficiency, achieving 88.10\% and 74.80\% test accuracy, respectively. Ablation studies validated the impact of quantum components on performance, with notable improvements in background rejection rates over classical counterparts. These results highlight Lorentz-EQGNN’s potential for immediate applications in noise-resilient jet tagging, event classification, and broader data-scarce HEP tasks. 
\end{abstract}

\begin{IEEEImpStatement}
The Lorentz-Equivariant Quantum Graph Neural Network (Lorentz-EQGNN) presented in this study offers a significant leap in quantum machine learning for high-energy physics (HEP) applications. Designed to handle unprecedented data volumes expected from the Large Hadron Collider, Lorentz-EQGNN effectively addresses limitations in current quantum graph neural networks (GNNs), such as fixed symmetry constraints and sensitivity to noise. By embedding Lorentz symmetry directly within a quantum GNN architecture, our model achieves robust performance in particle interaction modeling with only 4 qubits and a minimal parameter count. Unlike classical models and other quantum GNNs, Lorentz-EQGNN is symmetry-adaptive, enabling it to generalize well in data-scarce environments typical of HEP. Our approach outperforms existing models in quark-gluon jet tagging and electron-photon discrimination, offering a practical solution for critical HEP tasks with fewer computational resources. This adaptability to complex symmetries demonstrates Lorentz-EQGNN's potential to enhance applications across experimental physics, particularly in tasks that benefit from noise resilience and symmetry preservation. The model’s success suggests broad applicability beyond HEP, setting a new standard for efficient, symmetry-preserving AI in particle physics and potentially influencing the design of future quantum architectures.
\end{IEEEImpStatement}

% === KEYWORDS ====================================================================
% =================================================================================

% \fi
\IEEEpeerreviewmaketitle

% ====================================================================
% ====================================================================
% ====================================================================

% === I. INTRODUCTION =============================================================
% =================================================================================

\section{Introduction}
\IEEEPARstart{T}he unprecedented data volumes expected from the high-luminosity stage of the Large Hadron Collider (LHC)~\cite{apollinari_high-luminosity_2015} represent a significant computational challenge for particle physics research. Addressing this demand necessitates advances in computational efficiency to effectively manage and analyze such vast datasets~\cite{apollinari_high-luminosity_2015}. Quantum machine learning (QML) has emerged as a promising approach, offering the potential to lower classical algorithm time complexity through the unique advantages of quantum computing, specifically leveraging the exponentially large Hilbert space accessible on quantum hardware~\cite{schuld_quantum_2019}. In particle physics, symmetries form the foundation of our understanding of subatomic interactions. For instance, the standard model (SM) is built upon gauge and Lorentz invariance principles, which define how particles and fields interact within spacetime~\cite{gross_role_1996}. In parallel, symmetry is equally significant in machine learning, where it has been shown that a neural network invariant under a specific group can represent outputs as functions of the group’s orbits. This concept has catalyzed the development of the field of geometric deep learning~\cite{bronstein_geometric_2021}, which suggests that deep learning’s success in high-dimensional spaces, despite the curse of dimensionality, can be attributed to two primary inductive biases: symmetry and scale separation. These inductive biases streamline the hypothesis space, enabling models to become more compact and data-efficient, qualities especially advantageous in quantum computing given current hardware constraints.

The deep connection between symmetries and deep learning has driven the development of numerous models designed to be invariant to specific groups. For instance, convolutional neural networks (CNNs) excel in computer vision due to their natural invariance to image translations, while transformer-based language models demonstrate permutation invariance. When provided with suitable data, these architectures effectively learn stable latent representations aligned with their respective symmetry groups. In QML, leveraging symmetries presents a promising avenue, particularly in particle physics~{\cite{forestano_deep_2023,forestano_comparison_2024}}. Given that collision event data is often represented as graphs, graph neural networks (GNNs)~{\cite{velickovic_everything_2023}} are a fitting choice for applications in particle physics~{\cite{shlomi_graph_2020}}. Additionally, Lorentz symmetry plays a foundational role as a spacetime symmetry in any relativistic model of elementary particles. A robust approach to achieve invariance under any Lorentz transformation is via equivariant neural networks, where each layer’s output transforms accordingly if the input undergoes a Lorentz transformation. This setup allows stacking equivariant layers to build a symmetry-preserving deep network, ensuring that the final classification remains unchanged under Lorentz transformations. LorentzNet~{\cite{gong_efficient_2022}} represents the leading approach for incorporating Lorentz symmetry into GNNs, with applications such as jet tagging—identifying the particle responsible for generating a jet; however, key challenges remain in achieving data and computational efficiency in terms of parameter count, layer depth, and size of input data. Moreover, current QML models struggle with noise, which is a major issue in quantum systems, as hardware-induced noise can disrupt symmetry~{\cite{dong_2_2024}}. Moreover, most quantum GNNs rely on fixed symmetry groups, making them less flexible in handling the complex and changing symmetries of particle interactions.

To address these challenges, we aim to develop an approach that learns symmetries directly from the data. QML presents a promising alternative in this context. While quantum computing hardware is still in its early stages, extensive research is dedicated to designing algorithms that harness its capabilities. Quantum algorithms are recognized for providing computational benefits over classical methods for certain tasks, like Shor’s algorithm, which can break down numbers significantly more efficiently than conventional techniques~\cite{wan_quantum_2017}. Additionally, studies indicate that QML could yield significant computational speedups~\cite{larocca_theory_2023}. By leveraging quantum computing’s potential to capture quantum correlations naturally, our research seeks to build a more resilient, symmetry-adaptive quantum GNN model. This design aims to overcome issues of noise and rigidity, enhancing model performance while reducing both data requirements and computational resources.

We introduce the \ul{Lorentz}-\ul{E}quivariant \ul{Q}uantum \ul{G}raph \ul{N}eural \ul{N}etwork (Lorentz-EQGNN), a model engineered to adapt to symmetries defined by any Lie algebra. This innovative approach facilitates flexible handling of symmetries while preserving their properties even in noise, making it particularly well-suited for high-energy physics (HEP) applications. The 4-vectors input are scalarized by Lorentz-EQGNN directly to obtain Lorentz symmetry by representing jets as a particle cloud. In particular, we employ Minkowski dot product attention, which aggregates the four vectors using weights derived from Lorentz-invariant geometric quantities based on the Minkowski metric. The universal approximation theorem for Lie equivariant mappings is considered in the design of Lorentz-EQGNN, which guarantees both universality and equivariance in its functionality. The significant advantage of Lorentz-EQGNN arises from replacing the multilayer perceptron (MLP) components of LorentzNet with parameterized quantum circuits (PQCs), resulting in a more efficient architecture for both training and inference. This substitution also allows Lorentz-EQGNN to be smaller than LorentzNet, demonstrating the benefits of improved symmetry adaptation through our approach. We demonstrate our architecture on two problems from experimental HEP: quark-gluon jet tagging and electron-photon discrimination datasets, which is an ideal testing ground because experimental collider physics data are known to have a significant amount of complexity, and computational resources represent a major bottleneck~{\cite{larkoski_jet_2020}}. Lorentz-EQGNN is tailored for these HEP tasks with lower parameter and qubit counts, fewer data requirements, and shallow layer depth. Our model also shows better generalization on two generic datasets with two times less qubit size. Remarkably, its performance surpasses the classical benchmark, LorentzNet, in terms of accuracy and efficiency with even 6 times less layer depth despite the nascent stage of quantum computing. Our work focuses on the utility of quantum systems in the noisy intermediate-scale quantum (NISQ) era, emphasizing their practical advantages in speed, accuracy, and energy efficiency when compared to classical systems of similar scale and cost~{\cite{herrmann2023}}.

Our novel contributions in this research are 3-fold:
\begin{enumerate}
\item Our approach introduces a Lie-algebra adaptive quantum GNN that learns symmetry constraints from HEP data, improves data efficiency by introducing dimension reducer, thereby using the minimal number of qubits, and improves noise resilience.
\item We benchmark our Lorentz-EQGNN against a classical CNN, 3 quantum models, and two hybrid models on two HEP and two generic datasets and achieved competitive results with fewer data, parameters, circuit depth, and qubit requirements.
\item The ablation studies revealed the computational efficiency and superior accuracy of Lorentz-EQGNN components over classical state-of-the-art LorentzNet with six times fewer layer counts.
\end{enumerate}

The remainder of this paper is organized as follows: Section \hyperref[sec2]{II} reviews related work in classical and quantum GNNs, followed by Lorentz-equivariance. Section \hyperref[sec2]{III} details the architecture and theoretical foundations of our proposed Lorentz-EQGNN model and the other developed models for benchmarking. Section \hyperref[sec2]{IV} presents our experimental setup and datasets used for evaluation, followed by the experimental results and discussion in Section \hyperref[sec2]{V}. Finally, Section \hyperref[sec2]{VI} summarizes the key findings and outlines future research directions.

\section{Related Work}\label{sec2}
\subsection{Geometric Deep Learning and Transformer Models}
Recent advancements in geometric deep learning and transformer-based architectures have significantly contributed to HEP tasks. The Lorentz Geometric Algebra Transformer (L-GATr)~{\cite{spinner2024lorentzequivariant}} ensures Lorentz equivariance while modeling data in a geometric algebra framework, and the HEPT transformer~{\cite{siqi2024}} achieves efficient processing for large-scale point cloud data. SUPA simulator~{\cite{kumar2023}}, a lightweight diagnostic simulator, has been developed, connecting directly to geometric deep learning methods to address the computational challenges of particle shower simulations. However, existing methods struggle with noise resilience, fixed symmetry constraints, and data inefficiency, particularly under current hardware constraints and large data volumes in HEP. In the quantum domain, Unlu et al.~{\cite{unlu_hybrid_2024}} introduced hybrid quantum vision transformers for event classification, while Cara et al.~{\cite{cara2024}} applied quantum vision transformers to jet classification. Tesi et al.~{\cite{tesi2024qvit}} incorporated quantum attention mechanisms into vision transformers. These developments highlight the growing integration of geometric symmetries, efficient processing techniques, and quantum enhancements to address the complexities of HEP data, but further innovations are needed to enhance scalability, symmetry adaptability, and computational efficiency in resource-constraint environments.

\subsection{Classical Graph Neural Networks}
GNNs are a unique type of neural network created to handle data organized as graphs. In contrast to conventional neural networks, which operate on Euclidean data such as images or sequences, GNNs have the capability to model complex dependencies and relationships between the nodes of a graph. The core idea behind GNNs is an iterative update mechanism that refines each node's representation by aggregating information from its neighboring nodes. This method allows the network to understand both small-scale and large-scale patterns in the graph. 

Formally, let \( G = (V, E) \) denote a graph, with \( V \) representing the collection of nodes and \( E \) indicating the collection of edges. Each node \( v \in V \) is initially linked to a feature vector \( \mathbf{h}_v^{(0)} \). The features of the nodes are modified through several layers of the GNN. At each layer \( l \), the node's feature vector is modified by combining its own feature vector with those of its neighboring nodes \( \mathcal{N}(v) \). This can be expressed mathematically as:
\begin{equation}
\mathbf{h}_v^{(l+1)} = \sigma \left( \mathbf{W}^{(l)} \mathbf{h}_v^{(l)} + \sum_{u \in \mathcal{N}(v)} \mathbf{W}_e^{(l)} \mathbf{h}_u^{(l)} \right)
\end{equation}
where \( \mathbf{W}^{(l)} \) and \( \mathbf{W}_e^{(l)} \) are learnable weight matrices, \( \sigma \) is a non-linear activation function, and \( \mathcal{N}(v) \) represent the collection of adjacent nodes to \( v \). This iterative process continues through several layers, which allows the network to capture higher-order neighborhood information. The node representations obtained can be used for various tasks like classifying nodes, predicting links, and classifying graphs.

\subsection{Quantum Graph Neural Networks (QGNNs)}
With classical GNNs focusing on capturing the topological structure of graphs in low-dimensional representations, the focus has now shifted to QGNNs. Developing QGNNs in a methodical manner requires pinpointing a Hamiltonian that characterizes the qubits and their relationships in the graph. Utilizing the transverse-field Ising model~{\cite{forestano_comparison_2024}}, created to examine phase changes in magnetic systems, can lead to this achievement. The model is represented by the Hamiltonian as follows:
\begin{equation}\label{eq2}
\hat{H}_{\text{Ising}}(\boldsymbol{\phi}) = \sum_{(i, j) \in E} \mathcal{W}_{ij} Z_i Z_j + \sum_i \mathcal{M}_i Z_i + \sum_i X_i
\end{equation}
where \( \phi = \{\mathcal{W}, \mathcal{M}\} \), \( E \) denotes the set of all pairs of nodes, \( \mathcal{W}_{ij} \) is a learnable edge weight matrix, and the coupling term encapsulates the interactions between nodes through the coupled $Pauli-Z$ operators \( Z_i \) and \( Z_j \), which act on nodes \( i \) and \( j \) respectively. The first-order terms characterize the node weights, where each \( \mathcal{M}_i \) represents a learnable feature corresponding to the node's weight, effectively measuring the energy of the individual qubits. To utilize the Hamiltonian within a quantum computing framework, it is necessary to convert it into a unitary operator. This conversion is achieved through the time-evolution operator:
\begin{equation}
U = e^{-itH}
\end{equation}
which, when applied to the Ising Hamiltonian, takes the form:
\begin{equation}
U = e^{-it\hat{H}_{\text{Ising}}} = e^{\left[-i t \left( \sum_{(i, j) \in E} \mathcal{W}_{ij} Z_i Z_j + \sum_i \mathcal{M}_i Z_i + \sum_i X_i \right)\right]}
\end{equation}
However, implementing this expression directly on quantum hardware presents significant challenges. By applying the Trotter-Suzuki decomposition~\cite{hatano_finding_2005}, an approximation is obtained that is more feasible for quantum execution:\vskip-\parskip
\footnotesize
\begin{equation}
e^{\left[-i t \left( \sum_{(i, j) \in E} \mathcal{W}_{ij} Z_i Z_j + \sum_i \mathcal{M}_i Z_i + \sum_i X_i \right)\right]} \approx \prod_{k=1}^{t/\Delta} \left[ \prod_{j=1}^Q e^{-i \Delta \hat{H}_{\text{Ising}}^j(\boldsymbol{\phi})} \right]
\end{equation}
\normalsize
Thus, the final expression for the unitary operator to be implemented on quantum hardware is given by: $\hat{U}_{\mathrm{QGNN}}(\boldsymbol{\eta}, \boldsymbol{\phi}) = \prod_{p=1}^P \left[ \prod_{q=1}^Q e^{-i \eta_{pq} \hat{H}_q(\boldsymbol{\phi})} \right]$, where \( \eta_{pq} \) signifies a small parameter, \( Q = 3 \) indicates the total number of summations in the Hamiltonian~\ref{eq2}, and \( P \) refers to the count of layers in the QGNN.

This approach naturally leads to a permutation-EQGNN based on the Ising Hamiltonian. As previously discussed, the Lorentz group governs quantum field symmetries, which capture the invariance of physical laws across different inertial frames. Existing studies have proposed efficient methods to achieve equivariance in such systems. For example, {\cite{nguyen_theory_2024}} highlights that an equivariant model can be constructed using a set of quantum gates designed to preserve symmetry. However, these frameworks typically rely on the assumption that the underlying symmetry group is either discrete or a compact Lie group. Since the Lorentz group is noncompact, we employ an alternative approach rooted in the Invariant Theory {\cite{villar2023}}, diverging from the strategies outlined in {\cite{nguyen_theory_2024}} for achieving equivariance. Furthermore, our methodology for QGNNs is distinct from those introduced in {\cite{forestano_comparison_2024}}. The principles underlying our approach have been successfully applied in classical contexts, particularly in HEP {\cite{gong_efficient_2022}}, which is the primary inspiration for our work.

\subsection{Invariant-Equivariant Machine Learning}
In the context of Lie groups~\cite{georgi_lie_2019}, as applied in particle physics, our main objective is to explore how these symmetries can be incorporated into hybrid quantum-classical machine learning architectures. Mathematically, we treat the data as vectors \( x \) sampled from an input space \( X = \mathbb{R}^n \). A group element \( g \in G \) acts on \( x \) through a linear transformation \( T_X(g) \), where \( T_X: G \to \text{GL}(n) \) is the group representation. Essentially, \( T_X \) maps each element of the group \( G \) to an invertible matrix \( T_X(g) \in \mathbb{R}^{n \times n} \). Each representation of a Lie group also defines a corresponding representation in its Lie algebra.

The concept of symmetry-preserving learning is essential in machine learning, as it simplifies the learning task by reducing the effective hypothesis space. Models that incorporate geometric properties as an inductive bias are naturally more data-efficient, which can significantly reduce the need for data augmentation~\cite{lim_what_2022,satorras_en_2021,maron_invariant_2018,gong_efficient_2022,ecker_rotation-equivariant_2018,forestano_discovering_2023,forestano_accelerated_2023,forestano_deep_2023}. In the coming age of quantum machine learning (QML), where qubit resources are limited, such symmetry-preserving properties are even more critical~\cite{forestano_comparison_2024,dong_2_2024}.

Given a function \( f: \mathcal{X} \to \mathcal{Y} \) (where \( f \) could represent a neural network, \( \mathcal{X} \) the input space, and \( \mathcal{Y} \) the latent space), and a group \( G \) that acts on both \( \mathcal{X} \) and \( \mathcal{Y} \), the function \( f \) is said to exhibit equivariance if \( \forall g \in G, (x, y) \in \mathcal{D}, T_y (g) y = f(T_x (g) x) \). This relationship can be more compactly written as:
\begin{equation}
g \cdot y = f(g \cdot x)
\end{equation}
Invariance is a particular type of equivariance defined by the equation \( f(g \cdot x) = f(x) \), which indicates that \( T_y(g) = \mathbb{I} \); therefore, any group element acting on the space \( \mathcal{Y} \) results in the identity transformation. To incorporate invariance into a model, identifying the group \( G \) that encapsulates the transformations present in the data is crucial. In this context, Lie groups hold significant importance. We focused our analysis on the Lorentzian group, which plays a key role as a symmetry in the standard model of particle physics.

\begin{figure*}[!ht]
    \centering
    \includegraphics[width=1\linewidth]{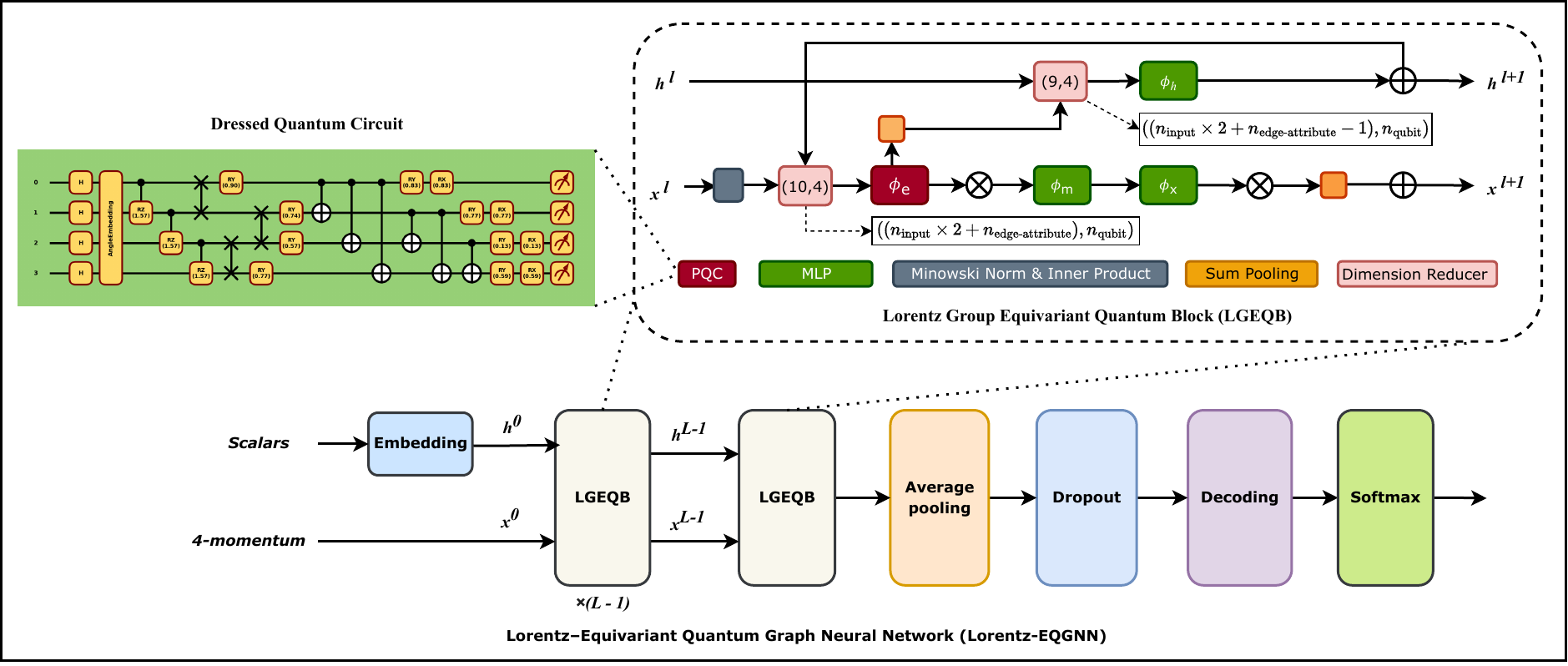}
    \caption{Proposed architecture of Lorentz-EQGNN. Within LGEQB, we replace standard $\phi$ modules with customized MLPs and PQCs. Here, only the MLP replacement in $\phi_e$ is shown. Dimension reducers are applied before $\phi_e$ and $\phi_m$ to ensure input compatibility across LGEQB, minimizing qubit requirement.}
    \label{fig:diagram}
\end{figure*}

\section{Methodology}\label{sec3}
\subsection{Equivariant Quantum Graph Neural Networks (EQGNN)}
QML offers a novel approach for incorporating equivariance in neural networks that leverage group symmetries. For a specific group \( \mathcal{G} \), equivariance~\cite{nguyen_theory_2024} can be represented through the construction of a quantum neural network defined as \( h_{\phi} = \text{Tr}[\sigma \tilde{M}_{\phi}] \), where \( h_{\phi} \) is the output of the quantum neural network and \( \sigma \) represents the input quantum state or density matrix from the input domain \( \mathcal{H} \). \( \tilde{M}_{\phi} \in Comm(G) = \{ B \in \mathbb{C}^{d \times d} \mid [B, R(g)] = 0 \text{ for all } g \in G \} \), where \( Comm(G) \) is the commutant of \( G \) and consists of matrices that commute with all representations \( R(g) \) of elements \( g \) $\in$ \( G \). This commutativity condition ensures that the model exhibits equivariance, leveraging the cyclical property of the trace operation. Specifically, for an element \( g \) in the group \( G \) and its representation \( R(g) \), we observe:\vskip-\parskip
\footnotesize
\begin{multline}
h_{\phi}(g \cdot \sigma) = \text{Tr}[R(g) \sigma R^{\dagger}(g) \tilde{M}_{\phi}] = \text{Tr}[\sigma R^{\dagger}(g) \tilde{M}_{\phi} R(g)] = \\
\text{Tr}[\sigma \tilde{M}_{\phi}] = h_{\phi}(\sigma)
\end{multline}\normalsize
where \( g \cdot \sigma \) represents the transformed input state defined by the adjoint action \( R(g) \sigma R^{\dagger}(g) \). \( \text{Tr}[\cdot] \) denotes the trace operation used for calculating the output. \( R(g) \) is a unitary representation of the group element \( g \) on the Hilbert space \( \mathcal{H} \), and \( R^{\dagger}(g) \) is its Hermitian conjugate.

In the Heisenberg context, where the evolution is applied to the measurement operator rather than the quantum state, invariance is achieved when the observable $\tilde{M}_{\phi}$ is part of the commutant of $\mathcal{G}$. However, this framework is restricted to finite-dimensional and compact groups such as $p4m$ and $SO(3)$. Continuous and non-compact groups, such as the Lorentz group, present challenges due to the absence of finite-dimensional unitary representations. We can integrate equivariance into the feature space and the message-passing mechanism to address this issue instead of embedding it directly into the ansatz. When the input remains invariant, the message-passing operation can naturally become equivariant. Equivariant GNNs (EGNNs) can also be extended to other significant groups, such as the Lorentz group, which plays a crucial role in HEP~{\cite{gong_efficient_2022}}. Particles detected at LHC move at speeds approaching that of light, requiring descriptions framed within the context of special relativity. Mathematically, the transformations between two inertial reference frames are described by the Lorentz group $O(1, 3)$.

\subsection{Proposed Lorentz-EQGNN}
Lorentz-EQGNN employs a hybrid quantum-classical architecture where classical components handle data preprocessing, feature extraction, and result aggregation, while quantum circuits focus on symmetry-adaptive computations using parameterized quantum layers. This division of tasks leverages classical efficiency for managing large-scale datasets and quantum expressiveness for preserving symmetries and learning geometric invariants. The classical components reduce the dimensionality of input features to ensure compatibility with quantum layers, minimizing the number of qubits required. Conversely, the quantum circuits perform operations like entanglement and parameterized rotations to exploit the full Hilbert space for feature representation. Ablation studies (Subsection {\ref{ablation}}) confirm that the hybrid architecture outperforms purely quantum or classical models, achieving a superior balance between computational efficiency and noise resilience, especially under data-scarce conditions. This synergy highlights the practical advantages of hybrid approaches in overcoming quantum hardware limitations while maintaining high performance.

Following the methodology of LorentzNet~{\cite{gong_efficient_2022}}, our input comprises 4-momentum vectors and related particle scalars (e.g., color and charge). We build upon the LorentzNet architecture while introducing three key modifications (see Fig.~{\ref{fig:diagram}}): (1) extracting the invariant metric from the learned algebra, (2) replacing classical MLP-based modules ($\phi_e$, $\phi_x$, $\phi_h$, and $\phi_m$) with PQCs, whose architecture is shown in Fig.~{\ref{fig:phi_e}}, and (3) introducing two dimension reducers preceding $\phi_e$ and $\phi_h$, optimizing qubit usage for efficient encoding of input data from classical to quantum domains. Table~{\ref{tab:1}} outlines module-specific adjustments, highlighting the classical-quantum layer replacements and indicating parameters configured for each quantum layer type.

\begin{table}[!ht]
\caption{Modifications of the LorentzNet modules present in the LGEQB block of Lorentz-EQGNN. Here, n\_hidden and n\_output are dimensions of latent space equal to 4}
\label{tab:1}
\resizebox{\columnwidth}{!}{%
\begin{tabular}{|l|l|l|}
\hline
\textbf{Module} & \textbf{LorentzNet}                       & \textbf{Lorentz-EQGNN}                                         \\ \hline
\multirow{5}{*}{$\phi_e$} & Linear(n\_hidden, n\_hidden, bias=False) & \multirow{5}{*}{DressedQuantumNet(n\_input)}        \\
         & Batch Normalization 1D           &                                                   \\
         & ReLU Activation                            &                                                   \\
         & Linear(n\_hidden, n\_hidden)     &                                                   \\
         & ReLU Activation                            &                                                   \\ \hline
\multirow{4}{*}{$\phi_h$} & Linear(n\_hidden, n\_hidden)             & \multirow{4}{*}{Dressed Quantum Circuit(n\_hidden)} \\
         & Batch Normalization 1D           &                                                   \\
         & ReLU Activation                            &                                                   \\
         & Linear(n\_hidden, n\_output)     &                                                   \\ \hline
\multirow{3}{*}{$\phi_m$}     & Linear(n\_hidden, 1)                     & Dressed Quantum Circuit(n\_hidden)                  \\
         & \multirow{2}{*}{Sigmoid Activation}         & Linear(n\_hidden, 1)                              \\
         &                                  & Sigmoid Activation                                          \\ \hline
\multirow{3}{*}{$\phi_x$}       & Linear(n\_hidden, n\_hidden)             & Dressed Quantum Circuit(n\_hidden)                  \\
         & ReLU Activation                            & \multirow{2}{*}{Linear(n\_hidden, 1, bias=False)} \\
         & Linear(n\_hidden, 1, bias=False) &                                                   \\ \hline
\end{tabular}%
}
\end{table}

\subsubsection{Input Layer}
The network receives the 4-momenta of particles from a collision event as input, which may also include associated scalars such as labels and charges. For the quark-gluon dataset, a jet can be represented as a graph by treating its constituent particles as nodes. For each particle \( j \), its 4-momentum vector \( p_{j,\alpha} = (E_{j,\alpha}, p_{j,\alpha}^x, p_{j,\alpha}^y, p_{j,\alpha}^z) \) defines the node’s coordinates in Minkowski space. Node attributes, such as mass, charge, and particle identity, are given by \( a_{j,\alpha} = (a_{j_1,\alpha}, a_{j_2,\alpha}, \dots, a_{j_n,\alpha}) \). Together, \( f_{j,\alpha} = p_{j,\alpha} \otimes a_{j,\alpha} \) encapsulates key tagging features. The scalars encompass the mass of the particle (calculated as \( (E_{j,\alpha})^2 - (p_{j,\alpha}^x)^2 - (p_{j,\alpha}^y)^2 - (p_{j,\alpha}^z)^2 \)) or, when available, particle identification (PID) information.

\begin{figure}[!ht]
    \centering
    \includegraphics[width=1\linewidth]{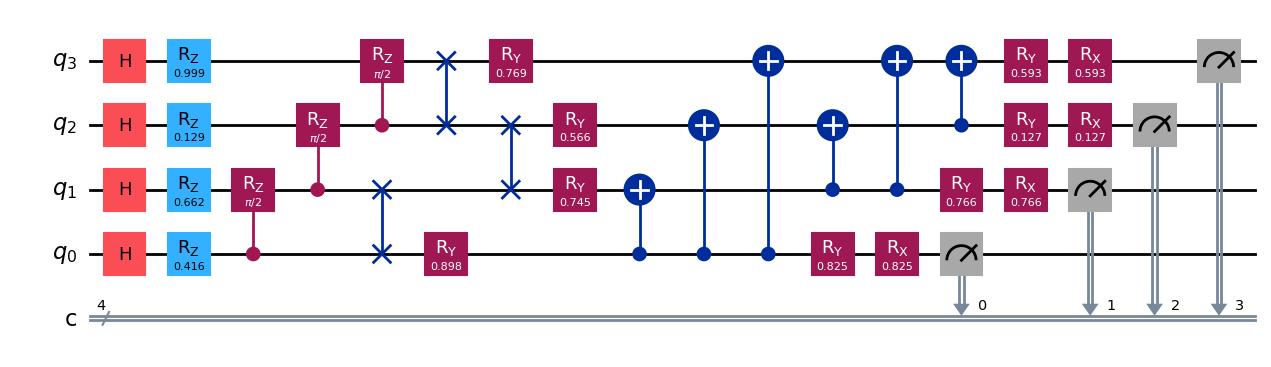}
    \caption{Dressed quantum circuit architecture for the LGEQB in the Lorentz-EQGNN framework. Here, the figure was generated for depth size = 2 to illustrate the varying entanglement operation in both even and odd depth sizes.}
    \label{fig:phi_e}
\end{figure}

\subsubsection{Dressed Quantum Circuit}
The architecture of the quantum circuit (see Fig. \ref{fig:phi_e} and Algorithm \ref{alg1}) implemented within the Lorentz-EQGNN comprises several strategically designed layers to enhance the model's expressiveness and representational capacity. The circuit initialization begins with a layer of Hadamard ($H$) gates applied to each of the \( n \) = 4 qubits, setting them into a superposition state. Mathematically, this transformation can be represented as:
\begin{equation}
H(q_i) = \frac{1}{\sqrt{2}} \begin{pmatrix} 1 & 1 \\ 1 & -1 \end{pmatrix}, \quad i \in \{0, \ldots, n-1\}
\end{equation}
This initialization creates an equal probability distribution of measurement outcomes, positioning the qubits at a balanced initial state crucial for subsequent operations.

Following the initialization, the circuit utilizes the \textit{angle embedding} technique to embed input features into the quantum state. This embedding is achieved through \( Z \)-axis rotations, defined by:
\begin{equation}
R_z(\psi_i) = \begin{pmatrix} e^{-i\psi_i / 2} & 0 \\ 0 & e^{i\psi_i / 2} \end{pmatrix}
\end{equation}
Here, \(\psi_i\) corresponds to feature values that are appropriately scaled to fit the range of the rotation. This step is critical for effectively encoding classical data into the quantum state.

The circuit then proceeds to apply alternating parameterized rotation layers to each qubit. The first type of rotation, the $RY$ layer, introduces parameterized \( Y \)-axis rotations, expressed as:
\begin{equation}
R_y(\gamma_i) = \begin{pmatrix} \cos(\gamma_i/2) & -\sin(\gamma_i/2) \\ \sin(\gamma_i/2) & \cos(\gamma_i/2) \end{pmatrix}
\end{equation}
This equation, \(\gamma_i\), represents trainable parameters that allow the circuit to learn from data adaptively. The second type of rotation, the Combined $RY-RX$ layer, applies both \( Y \)- and \( X \)-axis rotations on each qubit using shared weights. This can be expressed as:

\tiny
\begin{equation}
R_y(\gamma_i) \cdot R_x(\gamma_i) = \\
\begin{pmatrix} \cos(\gamma_i/2) & -\sin(\gamma_i/2) \\ \sin(\gamma_i/2) & \cos(\gamma_i/2) \end{pmatrix} \cdot \begin{pmatrix} \cos(\gamma_i/2) & -i\sin(\gamma_i/2) \\ -i\sin(\gamma_i/2) & \cos(\gamma_i/2) \end{pmatrix}
\end{equation}\normalsize

The circuit incorporates two distinct entangling layers to enhance the entanglement between qubits further. The first, termed the \textit{full entangling layer}, applies controlled-NOT (CNOT) gates between all pairs of qubits \( (q_i, q_j) \), represented mathematically as:
\begin{equation}
\text{CNOT}(q_i, q_j) = \begin{pmatrix} 1 & 0 & 0 & 0 \\ 0 & 1 & 0 & 0 \\ 0 & 0 & 0 & 1 \\ 0 & 0 & 1 & 0 \end{pmatrix}
\end{equation}
This layer ensures complete entanglement across all qubits, which is essential for quantum circuits to leverage the full potential of quantum parallelism.

The second entangling layer, referred to as the \textit{shifted entangling layer}, employs Controlled-RZ ($CRZ$) gates with an angle of \(\pi / 2\) and SWAP gates to facilitate interactions among qubits. The $CRZ$ gates are expressed as:
\begin{equation}
\text{CRZ}(\pi / 2, q_i, q_{i+1}) = \begin{pmatrix} 1 & 0 & 0 & 0 \\ 0 & e^{-i\pi / 4} & 0 & 0 \\ 0 & 0 & e^{i\pi / 4} & 0 \\ 0 & 0 & 0 & 1 \end{pmatrix}
\end{equation}
These gates are followed by $SWAP$ gates, which exchange the states of adjacent qubits, thereby cycling entanglement across the register and enriching the quantum state representation.

Finally, the circuit concludes with measurement operations that compute the expectation values in the \( Z \)-basis for each qubit, yielding results in the range \([-1, 1]\). This measurement is represented as: $\langle Z \rangle_{q_i} = \langle \psi | Z_i | \psi \rangle$. These expectation values constitute the quantum output features, which are subsequently processed in the classical layers of the Lorentz-EQGNN.

\begin{algorithm}[!ht]
\caption{Dressed Quantum Circuit}
\label{alg1}
\begin{algorithmic}[1]
\Require $input\_features$: Batch of input features, $(batch\_size, input\_dim)$
\Require $n\_qubits$: Number of qubits
\Require $q\_depth$: Circuit depth
\Require $q\_delta$: Initial scaling factor for parameters
\Ensure $q\_out$: Z-basis expectation values, $(batch\_size, n\_qubits)$

\State Initialize: Quantum device with $n\_qubits$, parameters $q\_params \sim \mathcal{N}(0, q\_delta)$
\State Preprocess: Transform $input\_features$ to $q\_in = \tanh(input\_features) \cdot \frac{\pi}{2}$
\For{$i = 1$ \textbf{to} $batch\_size$}
    \State Reshape $q\_params$ to $(q\_depth, n\_qubits)$
    \State Apply $H$ gates on all qubits
    \State Embed $q\_in[i]$ using $Z$-rotation
    \For{$k = 1$ \textbf{to} $q\_depth$}
        \If{$k$ \textbf{is even}}
            \State Apply entangling layer
            \State Apply $RY$ rotations with $q\_weights[k]$
        \Else
            \State Apply full entangling layer
            \State Apply $RY$ and $RX$ rotations with $q\_weights[k]$
        \EndIf
    \EndFor
    \State Measure $\langle Z \rangle$ for each qubit, store in $q\_out[i]$
\EndFor
\Return $q\_out$
\end{algorithmic}
\end{algorithm}

\subsubsection{Lorentz Group Equivariant Quantum Block (LGEQB)}
At the heart of our model is the Lorentz Group Equivariant Quantum Block (LGEQB), which is designed to learn deeper quantum representations of states \(|\psi_{x}^{l+1}\rangle\) and \(|\psi_h^{l+1}\rangle\) based on their respective inputs \(|\psi_{x}^{l}\rangle\) and \(|\psi_{h}^{l}\rangle\):
\begin{equation}
|\psi_{x}^{l+1}\rangle = \mathcal{U}_{x^{l+1}}({x}^{l})|0\rangle, \quad |\psi_{h}^{l+1}\rangle = \mathcal{U}_{h^{l+1}}({h}^{l})|0\rangle
\end{equation}
where \(\mathcal{U}_{x^{l}}, \mathcal{U}_{h^{l}}\) are parameterized quantum gate unitaries or variational circuits. Initially, \(x^{l}\) represents the particle observables, and \(h^{l}\) denotes the scalar values. For subsequent layers (\(l > 0\)), these values are derived from the expectation values \(\langle \psi_x | \mathcal{M} | \psi_x \rangle\) or \(\langle \psi_h | \mathcal{M} | \psi_h \rangle\), where \(\mathcal{M}\) is a quantum measurement operator.

We represent the node embedding scalars as \( h^l = (h_1^l, h_2^l, \ldots, h_N^l) \) and the coordinate embedding vectors as \( x^l = (x_1^l, x_2^l, \ldots, x_N^l) \) in the \( l \)-th LGEQB layer. At layer \( l = 0 \), \( x_i^0 \) corresponds to the input 4-momenta \( v_i \) and \( h_i^0 \) signifies the scalar variable embeddings \( s_i \). The purpose of the LGEQB is to derive deeper embeddings \( h^{l+1} \) and \( x^{l+1} \) from the existing embeddings \( h^l \) and \( x^l \). The edge message passing between particle \(i\) and \(j\) in the Lorentz-EQGNN can be formulated as:
\begin{equation}
m_{ij}^l = \phi_e\left(h_i^l, h_j^l, \psi_n(\|x_i^l - x_j^l\|^2), \psi_n(\langle x_i^l, x_j^l \rangle)\right)
\end{equation}
where \( \phi_e(\cdot) \) is a neural network function that encodes the interaction information. The normalization function is defined as \( \psi_n(z) = \text{sgn}(z) \log(|z| + 1) \) to stabilize large inputs from wide-ranging distributions. This formulation includes the Minkowski dot product \( \langle x_i^l, x_j^l \rangle \) and the squared norm \( \|x_i^l - x_j^l\|^2 \) as key features for particle interaction. Next, the Minkowski dot product attention updates the coordinate embeddings as follows:
\begin{equation}\label{weight_decay}
x_i^{l+1} = x_i^l + c \sum_{j \in [N]} \phi_x(m_{ij}^l) \cdot x_j^l
\end{equation}
where \( \phi_x(\cdot) \in \mathbb{R} \) is a scalar function from neural networks, and \( c \) is a hyperparameter controlling the scale of \( x_i^{l+1} \) to prevent explosive growth. For the node embeddings, we have:
\begin{equation}
h_i^{l+1} = h_i^l + \phi_h\left(h_i^l, \sum_{j \in [N]} w_{ij} m_{ij}^l\right)
\end{equation}
where \( \phi_h(\cdot) \) is modeled by a neural network, with its output dimension matching that of \( h_{i}^{l+1} \). Additionally, we introduced a neural network \( \phi_m(\cdot) \) to learn the edge significance \( w_{ij} = \phi_m(m_{ij}^l) \in [0, 1] \), ensuring permutation invariance and facilitating handling varying numbers of nodes.

In the LGEB architecture, $n_{input}$ specifies the dimensionality of the input scalar embeddings, while $n_{hidden}$ defines the dimension of the latent space used to process the embeddings through successive layers. $(n_{input} \times 2 + n_{edge\_attribute})$ refers to the dimensionality of concatenated features used in the edge-wise message-passing mechanism, where $n_{edge\_attribute}$ denotes the additional features derived from Minkowski geometry (norms and dot products between coordinate embeddings of nodes) and potentially other attributes. By default, $n_{edge\_attribute}$ is set to 2 unless 10-dimensional features are included. We introduced two MLP-based dimension reducers to efficiently preprocess data and minimize qubit requirements for quantum layers. The first reducer processes $(n_{input} \times 2 + n_{edge\_attribute})$ features, where $(n_{input} \times 2)$ represents the concatenated scalar embeddings of two nodes ($h_i, h_j$), and $n_{edge\_attribute}$ includes relational features like the Minkowski norm ($\|x_i - x_j\|^2$) and the inner product ($\langle x_i, x_j \rangle$). It reduces this input to $n_{qubit}$, optimizing data for $\phi_e$. The second reducer processes $(n_{input} \times 2 + n_{edge\_attribute} - 1)$, omitting one edge feature, and reduces it to $n_{qubit}$ for $\phi_h$. The $n_{hidden}$ parameter is reused across the node update network $(\phi_h)$ and the scalar-message weighting network $(\phi_m)$. Furthermore, the dropout rate, $c_{weight}$ control regularization, and the strength of coordinate updates. The model's multi-layer structure allows iterative refinement of node embeddings and coordinates, culminating in graph-level predictions through pooling and classification layers. Although both $x_i^{l}$ and $h_i^{l}$ are updated through the layers, the final output only uses the scalar features $h_i^{L}$ from the last layer $L$. This strategy reduces redundancy and computational overhead because the $m_{ij}^{l}$ already incorporate information from both $x_i^{l}$ and $x_j^{l}$ and focusing on scalar features, simplifies the network output without losing critical information.

\subsubsection{Infrared and Collinear Safety}
Integrating infrared and collinear (IRC) safety guarantees that observables stay consistent when soft or collinear particles are present. As suggested in~\cite{nguyen_theory_2024}, an IRC-safe equivariant classical GNN outperforms conventional methods in tagging semi-visible jets from Hidden Valley models. The model must be invariant to infinitesimal or collinear emissions to achieve IRC safety. This can be ensured through message passing, such that:
\begin{multline}
\text{IR safety}: m^{l}(i,j) \rightarrow 0 \quad \text{as} \quad z \rightarrow 0, \text{and}\\
\text{C safety}: m^{l}(i,j + r) = m^{l}(i,j) + m^{l}(i,r) \quad \text{as} \quad \Delta_{jr} \rightarrow 0   
\end{multline}
\(z\) represents the energy of the emitted particle, \(r\) is a neighboring particle to \(j\), \(\Delta_{jr}\) is the distance or angular separation between particles \(j\) and \(r\). To maintain IR safety without compromising equivariance by \(z_j\), we modify the message function as follows:
\begin{equation}
m_{ij}^{l} = \frac{\langle x_i, x_j \rangle}{\sum_{k \in \mathcal{N}(j)} \langle x_i, x_k \rangle} \cdot \phi_e(h_i, h_j, \psi_n(\|x_{i}^{l} - x_{j}^{l}\|^2), \psi_n(\langle x_{i}, x_{j} \rangle))
\end{equation}
where $\mathcal{N}(j)$ denotes the neighboring particles of $j$, $\psi_n$ acts as a normalization factor designed to enhance training stability in the case of large distributions {\cite{gong_efficient_2022}}. The Minkowski inner product ensures that soft particles exert negligible influence, thereby preserving IR safety. Furthermore, this inner product remains invariant under Lorentz transformations, ensuring symmetry-preserving message propagation.

\subsubsection{Decoding layer}
After passing through \(L\) layers of LGEQB, the node embeddings \( h^L \) are decoded. Since \( h^L \) already includes prior information, we avoid redundancy by decoding only \( h^L \). We use average pooling to obtain \( h_{\text{avg}} = \frac{1}{N} \sum_{i \in [N]} h_i^L \), apply dropout to reduce overfitting, and then pass it through a two-layer decoding block with softmax for binary classification.

\subsubsection{Theoretical Analysis}
We now establish the theoretical foundation of our model.
\begin{proposition}
The coordinate embedding \(x^{l} = \{x_1^{l}, x_2^{l}, \dots, x_n^{l}\}\) exhibits Lie group equivariance, whereas the node embedding \(h^{l} = \{h_1^{l}, h_2^{l}, \dots, h_n^{l}\}\), which signifies the particle scalars, maintains Lie group invariance.
\end{proposition}
\begin{proof}
To demonstrate this, let \(Q\) denote a Lie group transformation. If the message \(m_{ij}^{l}\) is invariant under the action of \(Q\) for all \(i,j,l\), then \(x_{i}^{l}\) is naturally Lorentz group equivariant, as shown by the following equation:\vskip-\parskip\footnotesize
\begin{multline} 
Q \cdot x_i^{l+1} = Q\left(x_i^{l} + \sum_{j \in \mathcal{N}(i)} x_j^{l} \cdot \phi_x(m_{ij}^{l})\right) = \\
Q \cdot x_i^{l} + \sum_{j \in \mathcal{N}(i)} Q \cdot x_j^{l} \cdot \phi_x(m_{ij}^{l})
\end{multline}\normalsize
Here, \(Q\) acts through matrix multiplication, implying that applying \(Q\) externally is equivalent to applying it internally to the node embeddings from the previous layer. The invariance of \(m_{ij}^{l}\) ensures that the induced metric and inner product are also invariant under \(Q\), leading to:
\begin{multline} 
\|x_{i}^{0} - x_{j}^{0}\|^2 = \|Q \cdot x_{i}^{0} - Q \cdot x_{j}^{0}\|^2, \quad \langle x_{i}^{0}, x_{j}^{0} \rangle = \langle Q \cdot x_{i}^{0}, Q \cdot x_{j}^{0} \rangle
\end{multline}
Since \(m_{ij}^{l+1} = \phi_e(h_i^{l}, h_j^{l}, \|x_{i}^{l} - x_{j}^{l}\|^2, \langle x_{i}^{l}, x_{j}^{l} \rangle)\), the invariance of the norm and inner product ensures that the message passing mechanism remains symmetry-preserving. Furthermore, since \(h_i^{l+1} = h_i^{l} + \phi_h\left(h_i^{l}, \sum_{j \in \mathcal{N}(i)} w_{ij} m_{ij}^{l}\right)\), the invariance of \(h_i^{l+1}\) is maintained recursively across layers.
\end{proof}

\subsubsection{Training and Optimization Procedure}
We implemented 5-fold cross-validation for model training, employing the AdamW optimizer with a weight decay ($c$ in Eq.~\ref{weight_decay}) of 0.01 to reduce the cross-entropy loss. Training the Lorentz-EQGNN involves 35 epochs, beginning with a linear warm-up period of 5 epochs to obtain an initial learning rate of \(1 \times 10^{-3}\). After this, a cosine annealing learning rate schedule is implemented for the subsequent 50 epochs, with parameters \(T_0 = 4\) and \(T_{\text{mult}} = 2\). The final three epochs are subjected to a weight decay of \(\gamma = 1 \times 10^{-2}\). Quantum circuit training faces challenges such as vanishing gradients (barren plateaus), parameter initialization, and hardware noise. To address these, parameters are initialized near zero with a small scaling factor $(q\_delta)$ to prevent barren plateaus, while gradients are computed efficiently using the PyTorch-Pennylane interface and parameter-shift rules. Hyperparameter tuning is performed via grid search, optimizing circuit depth $(q\_depth)$ for a balance between expressiveness and computational cost, rotation strategies for improved feature representation, and entanglement methods to maximize qubit interaction. Batch-wise quantum execution is vectorized to minimize runtime overhead while maintaining scalability. After every training epoch, the model is evaluated on the test dataset, and the version that demonstrates the highest validation accuracy is chosen as the best model for the final testing stage. 
% Codes and models are openly available at \href{https://github.com/Abrar2652/Lorentz-EQGNN}{https://github.com/Abrar2652/Lorentz-EQGNN}.

\subsection{Baseline Models}
For benchmarking, we implemented a classical CNN, 3 quantum models (EQCNN, reflection EQNN, and $p4m$ EQNN), and 2 hybrid models (QNN and Hybrid EQCNN).

\subsubsection{Classical CNN}
We employed a conventional convolutional neural network (CNN) model that processes input images of dimensions \(16 \times 16\) with a single channel (grayscale). Filters of dimensions 32 and 64 with a size of \(3 \times 3\) are employed and activated by ReLU, and max pooling is applied to reduce the output dimensions. The feature map is then flattened and input into a fully connected \textit{Dense} layer with 128 units, followed by a \textit{Dropout} layer with a dropout rate of 0.5. The quark-gluon and electron-photon datasets include a final \textit{Dense} layer with a single unit and sigmoid activation. The Adam optimizer and binary cross-entropy loss function are used by the model.

We constructed a steerable CNN model for the MNIST and Fashion-MNIST datasets, exhibiting equivariance to the cyclic group \(C_8\) to manage rotations. The process begins with grayscale images as scalar fields, employing group-equivariant convolutions on regular fields across the network. After each convolution, antialiased pooling is implemented to gradually diminish spatial dimensions, followed by group pooling to extract rotation-invariant characteristics. The ultimate classification is performed with two fully connected layers. The cross-entropy loss is utilized by the model, and the Adam optimizer is employed, with a learning rate of \(5 \times 10^{-5}\) and a weight decay of \(1 \times 10^{-5}\).

\subsubsection{Fully Quantum EQCNN}
Equivariant quantum CNN (EQCNN) embeds input data into an 8-qubit quantum system using an \textit{equivariant-amplitude} embedding function. The embedding maps classical data \(x_i\) to quantum states \(\lvert \psi(x_i) \rangle\), ensuring data symmetries are preserved:
\begin{equation}
\lvert \psi(x_i) \rangle = U_{\text{embedding}}(x_i) \lvert 0 \rangle^{\otimes n}
\end{equation}
Each convolutional layer employs a unitary operation \(U_2\), parameterized as: $U_2(\psi) = e^{\left(i \sum_j \psi_j P_j\right)}$, where \(\psi_j\) are the learnable parameters and \(P_j\) are Pauli operators acting on qubits. For each layer, 6 parameters are used, optimizing over multiple qubits and pairs. The pooling layers perform qubit reduction, defined as: $\text{Pooling}(q_1, q_2) = U_{\text{pool}} \lvert q_1 q_2 \rangle \rightarrow \lvert q_1 \rangle$. The network's structure comprises 5 layers, alternating convolutional and pooling layers. The quantum circuit outputs are measured using $PauliZ$ operators for qubits 4 through 7. The model is trained using the mean squared error (MSE) loss function. Training is performed over 25 epochs on 2000 samples using Nesterov optimization with an initial learning rate of 0.5, decaying by 0.5 every 10 steps.

\subsubsection{Hybrid QNN}
The hybrid quantum neural network (QNN) consists of a quantum node integrated with classical layers. The model accepts input images, which are reshaped into a one-dimensional array. The quantum layer utilizes Amplitude Embedding to encode the input data into 8 qubits, followed by \textit{basic entangler layers} that apply entanglement operations. The quantum output is generated as expectation values from $Pauli-Z$ measurements of each qubit. This output is then processed by two classical linear layers: the first maps the quantum output to a hidden layer of 256 units, and the second reduces it to two output classes. A Softmax activation layer is applied to produce probabilistic outputs for classification. The model is optimized using stochastic gradient descent (SGD) with an $L1$ loss function.

\subsubsection{Hybrid EQCNN}
The hybrid EQCNN integrates an equivariant quantum circuit followed by a classical layer. We used an 8-qubit quantum circuit for the quantum convolutional layer, where classical data is encoded using Amplitude Embedding. The quantum operations involve parameterized rotations and entangling gates like $IsingZZ$ and $IsingYY$ ($U2$ block and Pooling Ansatz), applied across qubit pairs to capture quantum correlations. After the quantum layer, measurements are taken from specific qubits, with the resulting probabilities passed to a classical layer with $16 \times 2$ units and a softmax activation, allowing the model to output a final prediction.

\subsubsection{Reflection EQNN}
The reflection equivariant quantum neural network (REQNN) \cite{t_west_reflection_2023} leverages geometric quantum machine learning techniques to respect reflection symmetries in image data. Classical images \( x \in \mathcal{X} \) are encoded into a quantum state \( |\phi(x)\rangle \) through amplitude encoding. The quantum state evolves through a parameterized variational circuit \( U_{\phi} \), ensuring the model's equivariance under reflection symmetry. The action of the symmetry group \( \mathcal{G} \) on the encoded states is given by:
\begin{equation}
R_{g} |\phi(x)\rangle = |\phi(g(x))\rangle, \quad \forall g \in \mathcal{G}, x \in \mathcal{X}
\end{equation}
Predictions \( \hat{y}_{\phi}(x) \) are computed based on quantum measurements:
\begin{equation}
\hat{y}_{\phi}(x) = \arg\max_{j} \langle \phi(x) | U_{\phi}^{\dagger} M_j U_{\phi} | \phi(x) \rangle
\end{equation}
where \( M_j \) is the measurement observable associated with class \( j \). The model's predictions remain invariant under the action of the reflection symmetry group \( \mathcal{G} \).

\subsubsection{p4m EQNN}
The p4m EQNN \cite{chang_approximately_2023} model ensures rotational and reflectional invariance for $16\times16$ pixel images. Each image is scaled to \([0, 1]\) and transformed with a sine function to map pixel values into \([-1, 1]\), then normalized and embedded into a quantum state using Amplitude Embedding. The circuit comprises translationally invariant $U2$ and $U4$ gates, applied across qubit pairs and quartets to capture complex entanglements while preserving symmetry. A pooling ansatz layer further enforces symmetry by interacting with qubit pairs. In the 4-measurement mode, $H$ gates measure probabilities on selected qubits; in the 2-measurement mode, additional $U4$ and pooling ansatz layers enhance reflectional symmetry before measurement. This setup ensures the model's invariance to rotations and reflections, which is ideal for tasks requiring p4m-equivariance.

\begin{table*}[!ht]
\caption{Performance Comparison on Quark-Gluon Dataset}
\label{tab:qg}
\resizebox{\textwidth}{!}{%
\begin{tabular}{l c c c c c c c c c c c }
\hline
\textbf{Model} &
   \textbf{\# Qubits} &\textbf{Optimizer} &
  \textbf{Learning rate} &
  \textbf{Batch size} &
  \textbf{Epoch} &
  \textbf{Loss function} &
  \textbf{Training subset size} &
  \textbf{Training time (s)} &
  \textbf{Inference time (s)} &
  \textbf{Train accuracy} &
  \textbf{Test accuracy} \\ \hline
Classical CNN    &  -&Adam & $1\times10^{-3}$ & 32  & 50 & Binary Crossentropy & 8000 (80\%)   & $\approx151.43$ & $\approx0.47$   & 91.09\%$\pm$3.37\% & 57.38\%$\pm$3.12\% \\ 
Hybrid EQCNN     &  8&SGD  & $1\times10^{-1}$      & 64  & 20 & L1 Loss             & 10000 (6.2\%) & $\approx576.28$  & $\approx2.30$    & 49.82\%$\pm$2.99\% & 48.70\%$\pm$1.85\% \\ 
Hybrid QNN       &  8&SGD  & $1\times10^{-1}$      & 64  & 50 & L1 Loss             & 10000 (10\%)  & $\approx1616.88$ & $\approx6.92$   & 60.00\%$\pm$1.92\% & \textbf{60.00\%$\pm$1.55\%} \\ 
EQCNN            &  8&Adam & $1\times10^-3$ & 128 & 10 & MSE                 & 20000 samples & $\approx706.08$  & $\approx325.78$ & 50.00\%$\pm$1.74\% & 50.00\%$\pm$1.53\% \\ 
Reflection-EQNN &  8&Adam & $1\times10^{-2}$     & 50  & 50 & Map Loss            & 500 (full)    & $\approx1057.36$ & $\approx0.02$   & 53.00\%$\pm$3.13\% & 51.00\%$\pm$4.21\% \\ 
p4m-EQNN        &  8&Adam & $1\times10^{-2}$ & 50  & 50 & Map Loss            & 500 (subset)  & $\approx1622.59$ & $\approx0.03$   & 53.00\%$\pm$4.21\% & 51.00\%$\pm$4.72\% \\ 
\textbf{Lorentz-EQGNN}        &  \textbf{4} &AdamW & $1\times10^{-3}$     & 16 & 50 & Cross Entropy Loss           & \textbf{800 (subset)}  & $\approx499.89$ &  $\approx38.94$  & 75.12\%$\pm$0.89\% & \textbf{74.00\%$\pm$0.26\%} \\ \hline
\end{tabular}%
}
\end{table*}

\begin{table*}[!ht]
\caption{Performance Comparison on Electron-Photon Dataset}
\label{tab:ep}
\resizebox{\textwidth}{!}{%
\begin{tabular}{ l c c c c c c c c c c c }
\hline
\textbf{Model} &
   \textbf{\# Qubits}&\textbf{Optimizer} &
  \textbf{Learning rate} &
  \textbf{Batch size} &
  \textbf{Epoch} &
  \textbf{Loss function} &
  \textbf{Training subset size} &
  \textbf{Training time (s)} &
  \textbf{Inference time (s)} &
  \textbf{Train accuracy} &
  \textbf{Test accuracy} \\ \hline
Classical CNN    &  -&Adam & $1\times10^{-3}$ & 32  & 50 & Binary Crossentropy & 160000 (80\%) & $\approx832.22$  & $\approx2.90$    & 70.43\%$\pm$4.97\% & \textbf{70.28\%$\pm$3.25\%} \\ 
Hybrid EQCNN     &  8&SGD  & $1\times10^{-1}$      & 64  & 20 & L1 Loss             & 10000 (6.2\%) & $\approx577.01$  & $\approx2.46$   & 49.81\%$\pm$2.71\% & 48.70\%$\pm$2.89\% \\ 
Hybrid QNN       &  8&SGD  & $1\times10^{-1}$      & 64  & 50 & L1 Loss             & 20000 (10\%)  & $\approx1516.18$ & $\approx7.22$   & 61.00\%$\pm$2.23\% & 61.00\%$\pm$2.52\% \\ 
EQCNN            &  8&Adam & $1\times10^{-3}$ & 128 & 10 & MSE                 & 20000 samples & $\approx477.29$  & $\approx213.45$ & 48.00\%$\pm$2.76\% & 49.00\%$\pm$2.66\% \\ 
Reflection-EQNN &  8&Adam & $1\times10^{-2}$ & 50  & 50 & Map Loss            & 500 (full)    & $\approx983.74$  & $\approx0.02$   & 53.00\%$\pm$3.21\% & 51.00\%$\pm$4.09\% \\ 
p4m-EQNN        &  8&Adam & $1\times10^{-2}$ & 50  & 50 & Map Loss            & 500 (subset)  & $\approx1266.69$ & $\approx0.03$   & 53.00\%$\pm$2.47\% & 51.00\%$\pm$3.12\% \\ 
\textbf{Lorentz-EQGNN}        &  \textbf{4} &AdamW & $1\times10^{-3}$     & 16 & 50 & Cross Entropy Loss           & \textbf{800 (subset)}  & $\approx263.38$  &  $\approx20.8$  & 49.75\%$\pm$7.13\% & \textbf{67.00\%$\pm$5.91\%} \\ \hline
\end{tabular}%
}
\end{table*}

\begin{table*}[!ht]
\caption{Performance Comparison on MNIST Dataset}
\label{tab:mnist}
\resizebox{\textwidth}{!}{%
\begin{tabular}{ l c c c c c c c c c c c }
\hline
\textbf{Model} &
   \textbf{\# Qubits} & \textbf{Optimizer} &
  \textbf{Learning rate} &
  \textbf{Batch size} &
  \textbf{Epoch} &
  \textbf{Loss function} &
  \textbf{Training subset size} &
  \textbf{Training time (s)} &
  \textbf{Inference time (s)} &
  \textbf{Train accuracy} &
  \textbf{Test accuracy} \\ \hline
Classical CNN    &  -&Adam & $5\times10^{-5}$ & 64  & 10 & Cross Entropy Loss & 512 (10\%)    & $\approx106.82$ & $\approx8.01$   & 100.00\%$\pm$0.37\% & \textbf{99.86\%$\pm$0.44\%}  \\ 
Hybrid EQCNN     &  8&SGD  & $1\times10^{-1}$ & 64  & 20 & L1 Loss            & 10000 (10\%)  & $\approx623.35$ & $\approx8.05$   & 99.00\%$\pm$0.13\%  & 99.00\%$\pm$0.16\%  \\ 
Hybrid QNN       &  8&SGD  & $1\times10^{-1}$      & 64  & 20 & L1 Loss            & 10000 (10\%)  & $\approx657.69$ & $\approx9.89$   & 100.00\%$\pm$0.30\% & \textbf{100.00\%$\pm$0.17\%} \\ 
EQCNN            &  8&Adam & $1\times10^{-2}$ & 128 & 50 & MSE                & 12665 samples & $\approx3128.8$ & $\approx119.49$ & 90.00\%$\pm$1.14\%  & 91.00\%$\pm$1.78\%  \\ 
Reflection-EQNN &  8&Adam & $1\times10^{-2}$     & 100 & 10 & Map Loss           & 2000 (full)   & $\approx999.07$ & $\approx0.10$   & 92.00\%$\pm$0.63\%  & 92.00\%$\pm$0.47\%  \\ 
p4m-EQNN        &  8&Adam & $1\times10^-{2}$     & 50  & 10 & Map Loss           & 500 (subset)  & $\approx253.1$  & $\approx0.04$   & 93.00\%$\pm$0.68\%  & 86.00\%$\pm$1.74\%  \\ 
\textbf{Lorentz-EQGNN}        &  \textbf{4}  &  AdamW & $1\times10^{-3}$     & 32 & 50 & Cross Entropy Loss           & \textbf{800 (subset)}  & $\approx1349.53$ &  $\approx121.76$  & 85.40\%$\pm$2.28\% & 88.10\%$\pm$3.02\% \\ \hline
\end{tabular}%
}
\end{table*}

\begin{table*}[!ht]
\caption{Performance Comparison on FashionMNIST Dataset}
\label{tab:fashion-mnist}
\resizebox{\textwidth}{!}{%
\begin{tabular}{ l c c c c c c c c c c c }
\hline
\textbf{Model} &
\textbf{\# Qubits} & \textbf{Optimizer} &
  \textbf{Learning rate} &
  \textbf{Batch size} &
  \textbf{Epoch} &
  \textbf{Loss function} &
  \textbf{Training subset size} &
  \textbf{Training time (s)} &
  \textbf{Inference time (s)} &
  \textbf{Train accuracy} &
  \textbf{Test accuracy} \\ \hline
Classical CNN    &  -&Adam & $5\times10^{-5}$ & 64 & 10 & Cross Entropy Loss & 1200 (10\%)   & $\approx82.21$   & $\approx1.49$   & 99.25\%$\pm$0.36\% & \textbf{98.50\%$\pm$0.36\%} \\ 
Hybrid EQCNN     &  8&SGD  & $1\times10^{-1}$ & 64 & 20 & L1 Loss            & 10000 (10\%)  & $\approx558.45$  & $\approx7.25$   & 92.00\%$\pm$1.83\% & 92.00\%$\pm$1.16\% \\ 
Hybrid QNN       &  8&SGD  & $1\times10^-1$      & 64 & 20 & L1 Loss            & 10000 (10\%)  & $\approx741.72$  & $\approx5.47$   & 99.00\%$\pm$0.36\% & \textbf{99.00\%$\pm$0.33\%} \\ 
EQCNN            &  8&Adam & $1\times10^{-1}$ & 64 & 20 & L1 Loss            & 10000 (100\%) & $\approx619.76$  & $\approx116.34$ & 91.36\%$\pm$1.63\% & 90.90\%$\pm$1.83\% \\ 
Reflection-EQNN &  8&Adam & $1\times10^{-2}$     & 50 & 50 & Map Loss           & 500 (full)    & $\approx1089.34$ & $\approx0.14$   & 57.00\%$\pm$3.37\% & 68.00\%$\pm$3.19\% \\ 
p4m-EQNN        &  8&Adam & $1\times10^{-2}$     & 50 & 50 & Map Loss           & 500 (subset)  & $\approx1204.67$ & $\approx0.15$   & 53.00\%$\pm$2.47\% & 51.00\%$\pm$4.35\% \\ 
\textbf{Lorentz-EQGNN}        &  \textbf{4}  &  AdamW & $1\times10^{-3}$     & 16 & 50 & Cross Entropy Loss           & \textbf{800 (subset)}  & $\approx1278.31$ &   $\approx100.29$ & 74.96\%$\pm$3.64\% & 74.80\%$\pm$3.45\% \\ \hline
\end{tabular}%
}
\end{table*}

\begin{table*}[!ht]
\centering
\caption{Ablation Studies of Lorentz-EQGNN on Quark-Gluon Jet Dataset using 5-fold Cross-Validation. \textbf{Bold} indicates best performance.}
\label{tab:components}
% \resizebox{\textwidth}{!}{%
\begin{tabular}{ l | c c c c c c }
\hline
\textbf{LGEQB Components} & \textbf{$\phi_e$} & \textbf{$\phi_h$} & \textbf{$\phi_m$} & \textbf{$\phi_x$} & \textbf{Fully Quantum} & \textbf{LorentzNet} \\ \hline
\textbf{\# Params}     & 199    & 175    & 243    & 239    & 139    & 1088  \\ 
\textbf{\# Layers}     & 1    & 1    & 1    & 1    & 1    & 6  \\ 
\textbf{Test Accuracy} & 74.00\%$\pm$1.54\%   & 74.00\%$\pm$1.63\%   & 73.00\%$\pm$1.26\%   & \textbf{77.00\%$\pm$1.76\%}   & 46.00\%$\pm$1.88\%   & 74.00\%$\pm$1.45\%  \\ 
\textbf{Test AUC}      & \textbf{87.38\%$\pm$1.52\%} & 83.03\%$\pm$1.63\% & 84.78\%$\pm$1.34\% & 85.59\%$\pm$1.65\% & 78.74\%$\pm$1.27\% & 79.39\%$\pm$1.82\%   \\ 
\textbf{Test Loss}     & 64.65\%$\pm$1.62\% & 65.70\%$\pm$1.67\%  & 58.55\%$\pm$1.62\% & 56.94\%$\pm$1.73\% & 70.72\%$\pm$1.26\% & 62.69\%$\pm$1.78\%  \\ 
\textbf{$\frac{1}{\epsilon_B}$ (${\epsilon_S}=0.3$)}      &  80.5$\pm$7  & 55$\pm$4     & 45$\pm$5   & 26.50$\pm$3.50   & 18$\pm$2      & 25$\pm$6      \\ 
\textbf{$\frac{1}{\epsilon_B}$ (${\epsilon_S}=0.5$)}      & 23$\pm$4     & 11$\pm$2     & 26$\pm$4     & 13.25$\pm$2.25  & 11.50$\pm$2.50   & 12.50$\pm$6 \\ \hline
\end{tabular}%
% }
\end{table*}

\begin{table}
\centering
% \begin{shaded}
\caption{Qubit Scalability of Lorentz-EQGNN on Quark-Gluon Dataset}
\label{tab:scalability}
\begin{tabular}{l |ccc} \hline 
\textbf{$\phi_{e}$ Components}  &\textbf{$n_{qubit}$=4}& \textbf{$n_{qubit}$=6}&\textbf{$n_{qubit}$=8}\\ \hline 
\textbf{\# Params} &199  
& 393  & 651 \\ 
\textbf{\# Layers} &1  & 1  & 1  \\ 
\textbf{Test Accuracy}  &74.00\%  & \textbf{80.00\%} & 77.00\%  \\ 
\textbf{Test AUC}  &87.38\% & \textbf{90.11\%}  & 84.94\%  \\ 
\textbf{Test Loss}  &64.65\%  & 68.85\% & 53.94\%  \\ 
\textbf{$\frac{1}{\epsilon_B}$ (${\epsilon_S}=0.3$)}  &80.5  
& 160  & 44  \\ 
\textbf{$\frac{1}{\epsilon_B}$ (${\epsilon_S}=0.5$)}  &23  & 62  & 11  \\ 
\textbf{Training time (s)} &  $\approx499.89$ & $\approx800.29 $  &  $\approx1202.53$ \\
\textbf{Inference time (s)} &  $\approx38.94$ & $\approx61.74$  &  $\approx86.47$ \\ \hline
\end{tabular}
% \end{shaded}
\end{table}

\section{Experiments}\label{sec4}
\subsection{Datasets and Preprocessing}
In this study, we utilized four distinct datasets to evaluate the performance of our benchmarked models: MNIST, FashionMNIST, ECAL images of electrons and photons from the CMS experiment, and synthetic quark-gluon particle jets. For Lorentz-EQGNN, all datasets are split into training (80\%), validation (10\%), and test (10\%) sets, stratified by class, to maintain balance. 

\subsubsection{HEP Datasets}
We evaluated our models on two experimental HEP problems: quark-gluon jet tagging and electron-photon discrimination. These datasets are ideal for testing due to their complexity and the challenges posed by limited computational resources. We employ the \textit{Pythia8 Quark and Gluon Jets for EnergyFlow}~\cite{komiske_energy_2019} dataset, comprising 2 million jets equally partitioned into 1 million quark jets and 1 million gluon jets, to ascertain the origin of a specific jet as either a quark or a gluon. These jets were generated from LHC collisions with a total center-of-mass energy of \(\sqrt{s} = 14\) TeV and were selected to possess transverse momenta \(p_{T,\alpha}^{\text{jet}}\) between 500 and 550 GeV, with rapidities \(|y_{\text{jet}}| < 1.7\). For our investigation, we randomly picked \(N = 12,500\) jets, allocating the first 10,000 for training, the subsequent 1,250 for validation, and the final 1,250 for testing, yielding 4,982, 658, and 583 quark jets in each respective set. We consider the jet dataset as a compilation of point clouds, characterizing each jet as a graph \(G = \{V, E\}\), where \(V\) signifies the set of vertices and \(E\) denotes the edges. Each particle, represented as a node, is treated as a point in Minkowski space $\mathbb{R}^{1,3}$, adhering to the spacetime symmetries of special relativity. The number of nodes per jet varies, capturing the inherent stochasticity of particle collisions. This architecture is optimal for GNNs. In our configuration, each node includes the transverse momentum \(p_{T,\alpha}\), pseudorapidity \(\eta\), azimuthal angle \(\psi\), and other scalar-like attributes such as particle ID and mass for the corresponding constituent particles in the jet. Although the number of features remains constant, the number of nodes within each jet may vary. For this analysis, we focused on jets containing at least ten particles. Likewise, the Electron-Photon~\cite{barnard2017parton} dataset comprises $33 \times 33$ pixel images and poses a binary classification task to determine whether an electron or a photon initiated the event. Each image is converted to a graph where significant points form nodes with eight feature values, capturing energy, position, and angular information. Nodes are connected with edges based on their locations in the image, forming a fully connected graph for each sample. 

\subsubsection{Generic Datasets}
We also used two generic simple datasets for benchmarking. MNIST is a widely used benchmark for image classification tasks. It contains 70,000 grayscale images of handwritten digits, each with a resolution of $28 \times 28$ pixels, split into 60,000 images for training and 10,000 for testing. The dataset comprises 10 classes, representing digits from 0 to 9. Similarly, we utilized the FashionMNIST dataset with 70,000 grayscale images of $28 \times 28$ pixels with more complex visual features. It includes ten classes of clothing items such as T-shirts, trousers, bags, and dresses. Each image is processed to extract up to 10 significant points based on pixel intensity thresholds. A 4-dimensional feature vector (representing normalized $x$ and $y$ coordinates, intensity, and a zero placeholder) is created for each point. Fully connected edges are created between nodes within each image.

\subsection{Experimental Configuration}
The experimental setup utilized a combination of GPUs and CPUs to accelerate the training and testing of all the benchmarked models in this study. Hardware resources included two NVIDIA T4 GPUs (2560 CUDA cores, 16 GB each) and a Google TPU (8 v3 cores, 128 GB), ensuring efficient parallel computations for large datasets. An Intel Xeon CPU (4 vCPUs at 2 GHz, 18 GB RAM) also supported general preprocessing and computation tasks. We utilized EnergyFlow 1.3.2~{\cite{komiske_energy_2019}} to download and read the quark-gluon dataset. The framework for classical computation utilized Python 3.10.12 and PyTorch, while PQC development, execution, and visualizations were implemented using PyTorch Geometric 2.6.1, PennyLane 0.38.0, Qiskit 1.2.4, and Pennylane-Qiskit 0.38.1. Preprocessing steps involved standardizing input data formats to be compatible with both classical and quantum pipelines. A complete description of the experimental tool flow is available in the repository.

\section{Results and Discussion}\label{sec5}
\subsection{Performance Comparison}
Lorentz-EQGNN outperforms baselines in both accuracy and robustness. Detailed results are provided in Tables \ref{tab:qg}, \ref{tab:ep}, \ref{tab:mnist} and \ref{tab:fashion-mnist}. We presented the mean and standard deviation of the 5-fold cross-validation outcomes for training and testing. We employed standard metrics to assess performance, such as test accuracy, AUC, and background rejection $\frac{1}{\epsilon_B}$ at signal efficiencies of ${\epsilon_S}$ = 0.3 and 0.5. In this context, ${\epsilon_B}$ denotes the false positive rate, while ${\epsilon_S}$ represents the true positive rate. The background rejection metric holds significant importance in selecting the optimal algorithm, as it directly relates to the anticipated background contribution~\cite{gong_efficient_2022}.

Lorentz-EQGNN demonstrates robust performance even under data scarcity,  with only 800 training samples across all datasets using 2 times fewer qubits. To further validate and improve performance in data-scarce scenarios, we propose incorporating self-supervised learning techniques, such as contrastive learning, to leverage unlabeled data effectively. Additionally, symmetry-preserving data augmentation methods can enhance the model's generalization by creating diverse yet physically consistent training samples. Transfer learning from related HEP datasets offers another pathway to mitigate data limitations, enabling the model to adapt pre-trained representations to new tasks with minimal labeled data.

\subsubsection{Results for Quark-Gluon}
Lorentz-EQGNN achieves the highest test accuracy of 74.00\% (AUC: 87.38\%) on the quark-gluon dataset, significantly outperforming all the approaches (Table~\ref{tab:qg}). This gain comes with efficient training ($\approx500$s) and inference time (38.94s) only using 800 training subsets and half qubits compared to the others, indicating that Lorentz-EQGNN can efficiently capture the particle graph data. The significant gap between our test accuracy (74.00\%) and other quantum methods ($\approx 60\%$) demonstrates the advantage of incorporating Lorentz symmetry while avoiding the overfitting seen in the Classical CNN (91.09\% train vs 57.38\% test accuracy).

\subsubsection{Results for Electron-Photon}
For the Electron-Photon dataset (Table~\ref{tab:ep}), the Classical CNN achieves a slightly higher test accuracy (70.28\%) compared to Lorentz-EQGNN (67.00\%). However, Lorentz-EQGNN achieves this competitive performance (AUC: 68.20\% 	 $\frac{1}{\epsilon_B} (\epsilon_S = 0.3)$: 10 	 $\frac{1}{\epsilon_B} (\epsilon_S = 0.5)$: 8.33) using only 800 training samples with 4-qubit and single depth size, while the Classical CNN relies on 160,000 samples. Training efficiency per data point remains challenging, as Lorentz-EQGNN requires more time to train on each sample. This trade-off between data efficiency and computational cost highlights an important consideration in QML. While quantum approaches can extract meaningful features from limited data, the quantum processing overhead must be balanced against the benefits of reduced data requirements.

\subsubsection{Results for MNIST}
As shown in Table~\ref{tab:mnist}, traditional models excel on MNIST, with the Hybrid QNN achieving 100\% test accuracy and the Classical CNN reaching 99.86\%. While Lorentz-EQGNN shows competitive performance at 88.10\% (AUC: 94.35\%, $\frac{1}{\epsilon_B} (\epsilon_S = 0.3)$: 83.33, $\frac{1}{\epsilon_B} (\epsilon_S = 0.5)$: 35.71), using only a 4-qubit single PQC layer, highlighting its efficiency for generic tasks apart from HEP. The performance gap on this classical computer vision task is expected, as MNIST's pixel-based structure doesn't inherently benefit from Lorentz symmetry preservation. However, achieving reasonable accuracy with minimal quantum resources demonstrates the architecture's versatility beyond its primary physics applications.

\subsubsection{Results for FashionMNIST}
On the FashionMNIST dataset (Table~\ref{tab:fashion-mnist}), Classical CNN and Hybrid QNN show strong performance, achieving 98.50\% and 99.00\% test accuracy, respectively. Lorentz-EQGNN achieves 74.80\% accuracy (AUC: 83.20\%, $\frac{1}{\epsilon_B} (\epsilon_S = 0.3)$: 33.33, $\frac{1}{\epsilon_B} (\epsilon_S = 0.5)$: 12.50), requiring only 800 training samples and fewer qubits. The performance pattern here mirrors that of MNIST, but with a more pronounced gap, reflecting FashionMNIST's greater complexity. Interestingly, our model maintains consistent performance between training (74.96\%) and testing (74.80\%), suggesting robust generalization despite the reduced quantum resources and training data. This stability across different datasets indicates that our architecture's symmetry-preserving properties contribute to reliable learning, even when applied to non-physics domains.

\subsection{Ablation Studies}\label{ablation}
We obtained the results for the ablation studies by substituting $\phi_e$, $\phi_h$, $\phi_m$, and $\phi_x$ with the 4-qubit parameterized dressed quantum circuits one at a time, similar to how it's presented in Table \ref{tab:1}. This replacement approach enabled us to isolate and evaluate the quantum advantage of each module within the network architecture. We also examined the scenario where all four modules—$\phi_e$, $\phi_h$, $\phi_m$, and $\phi_x$—are replaced with quantum circuits, as well as with the fully classical model LorentzNet, which does not utilize any PQC~\cite{gong_efficient_2022}. Notably, the fully quantum version showed degraded performance, suggesting that a hybrid quantum-classical approach strikes a better balance between quantum advantages and classical stability.

The results for the quark-gluon dataset, spanning 50 epochs and employing 5-fold cross-validation, are summarized in Table \ref{tab:components}. Table \ref{tab:components} indicates that the Lorentz-EQGNN, equipped with the quantum-integrated $\phi_e$ module, attains superior performance on the quark-gluon dataset, especially for AUC and background rejection at $\epsilon$ = 0.3 and 0.5. The superior performance of the $\phi_e$ module suggests that quantum circuits are particularly effective at processing edge features, which naturally encode the fundamental particle interactions in the jet structure.  The ablation tests validate the efficacy of the Lorentz-EQGNN relative to its classical counterpart, obtaining around $\approx5.5$ times fewer parameters, a six-fold reduction in depth size, and exhibiting an enhancement in background rejection by roughly two to three times. This significant reduction in model complexity while maintaining superior performance indicates that quantum circuits can capture the underlying physics more efficiently than their classical counterparts, potentially by leveraging quantum mechanical principles inherent to particle physics problems.

\subsection{Scalability and Qubit Utilization}
The scalability of Lorentz-EQGNN was evaluated with 4, 6, and 8 qubit configurations (Table~{\ref{tab:scalability}}). Increasing to 6 qubits provides the best balance between performance and resource efficiency, significantly enhancing test accuracy, AUC, and background rejection compared to 4 qubits while keeping computational costs manageable. However, scaling to 8 qubits results in diminishing returns, with reduced accuracy and AUC, overfitting risks, and increased training and inference times, highlighting the trade-offs between richer quantum representations and hardware limitations.

We acknowledge the limitations of current quantum hardware, including constraints on qubit count, circuit depth, noise, gate fidelity, and decoherence. Lorentz-EQGNN is tailored for NISQ devices, leveraging 4-qubit circuits with alternating entanglement and parameterized rotation layers to balance expressiveness and mitigate decoherence. While scalability to more qubits is theoretically possible, it requires careful optimization to avoid diminishing returns. Future directions include advanced error mitigation, circuit pruning, and transitions to fault-tolerant quantum hardware with error correction mechanisms (e.g., surface codes), enabling scalability to larger datasets and complex symmetries.

\section{Conclusion and Future Work}\label{sec6}
We introduce a Lorentz-EQGNN tailored for HEP problems, leveraging quantum-enhanced computation for robust and symmetry-preserving performance. By implementing efficient message propagation, the Lorentz-EQGNN upholds Lie invariance, reduces parameter counts, and minimizes computational overhead, improving generalization, especially in data-scarce scenarios. Evaluations on a small subset of two HEP-specific tasks and two generic image classification tasks demonstrate its competitive performance using only a single depth 4-qubit quantum circuit. Our Lorentz-EQGNN outperforms its classical state-of-the-art counterpart, LorentzNet, while using significantly fewer parameters, LGEQB layers, circuit depth, and data samples. The integration of PQCs efficiently replaces MLPs through symmetry preservation, and Lorentz symmetry enhances relativistic invariance handling, while the hybrid architecture enables scalable performance on NISQ devices. The model’s adherence to IRC safety ensures that soft particle emissions have negligible effects, reinforcing its applicability to various particle physics challenges such as jet mass regression, full 4-momentum reconstruction, fast jet simulation, event classification, and pileup mitigation. The broader implications of this work extend beyond particle physics, as our hybrid quantum-classical architecture provides a practical blueprint for quantum-enhanced neural networks in domains where quantum mechanics plays a fundamental role, such as molecular dynamics and quantum chemistry. Future research may explore advanced optimizations to further reduce computational overhead and improve scalability on larger datasets. Adaptations to handle variations in node numbers and more complex graph structures will expand Lorentz-EQGNN's applicability to broader HEP challenges. Investigating enhanced error mitigation strategies and leveraging advancements in fault-tolerant quantum hardware will further strengthen the model's utility in both experimental and theoretical physics.

\ifCLASSOPTIONcaptionsoff
  \newpage
\fi



%\begin{thebibliography}{1}


\bibliographystyle{IEEEtran}
\bibliography{IEEEabrv,main}

\vfill


\end{document}